\documentclass{article} % For LaTeX2e

\title{Learning in the Presence of Corruption}
\date{}

\author{
Brendan van Rooyen$^{*, \dagger}$ \\
\and
Robert C. Williamson$^{*, \dagger}$ \\
\and
\vspace{-0.15in} \\
$^*$The Australian National University \qquad $^\dagger$National ICT Australia \\
{\footnotesize \texttt{\{ brendan.vanrooyen, bob.williamson \}@nicta.com.au} } \\
}

\usepackage{times}
\usepackage{hyperref}
\usepackage{url}

\usepackage{bbold,bm}
\usepackage{amsthm, amsmath, amssymb}
\usepackage[lined,boxed,commentsnumbered]{algorithm2e}
\usepackage{caption}
\usepackage{graphicx}
\usepackage{float}
\usepackage[all]{xy}
\usepackage{booktabs}

\usepackage[left=1.25in, top=1in, bottom=1in, right=1.25in]{geometry}
\newcommand{\RR}{\protect\mathbb{R}}
 
\newcommand{\EE}{\protect\mathbb{E}}
\newcommand{\PP}{\protect\mathbb{P}}
\newcommand{\alg}{\protect\mathcal{A}}

\newcommand{\obs}{\protect\mathcal{O}}
\newcommand{\ind}{\protect\bm{1}}
\newcommand{\Fclass}{\protect\mathcal{F}}

\newcommand{\Df}{\protect{D_{f}}}
\newcommand{\dist}{\raise.17ex\hbox{$\scriptstyle\mathtt{\sim}$}}

\newcommand{\risk}{\protect{\mathcal{R}}}
\newcommand{\Rbar}[1]{\protect{\underline{\risk}_{#1}}}

\newcommand{\relloss}{\protect{\Delta L}}

\newcommand{\ip}[2]{\langle #1, #2\rangle}

\newcommand{\riskL}[1]{\risk_{#1}}
\newcommand{\pred}[1]{[\![ #1 ]\!]}

\newcommand{\logdelta}{\log\left(\frac{1}{\delta}\right)}

\DeclareMathOperator*{\argmin}{arg\,min}

\newtheorem{theorem}{Theorem}
\newtheorem{lemma}{Lemma}
\newtheorem{definition}{Definition}

\newcommand{\TnoisyLabels}{\left(\begin{array}{cc}
1-\sigma_{-1} & \sigma_1 \\ 
\sigma_{-1} & 1-\sigma_1
\end{array} \right)}

\newcommand{\MnoisyLabels}{\frac{1}{1 - \sigma_{-1} - \sigma_1} \left( \begin{array}{cc}
1-\sigma_1 & -\sigma_{-1} \\ 
-\sigma_{1} & 1-\sigma_{-1}
\end{array} \right) }

\newcommand{\TnoisyLabelsmulticlass}{\left(\begin{array}{ccc}
1-\sigma & \frac{\sigma}{2} & \frac{\sigma}{2} \\ 
\frac{\sigma}{2} & 1-\sigma & \frac{\sigma}{2} \\ 
\frac{\sigma}{2} & \frac{\sigma}{2} & 1-\sigma
\end{array} \right)}

\newcommand{\MnoisyLabelsmulticlass}{\left(\begin{array}{ccc}
\frac{2-\sigma}{2 - 3\sigma} & \frac{-\sigma}{2-3\sigma} & \frac{-\sigma}{2-3\sigma} \\ 
\frac{-\sigma}{2-3\sigma} & \frac{2-\sigma}{2 - 3\sigma} & \frac{-\sigma}{2-3\sigma} \\ 
\frac{-\sigma}{2-3\sigma} & \frac{-\sigma}{2-3\sigma} & \frac{2-\sigma}{2 - 3\sigma}
\end{array} \right)}

\newcommand{\Tsemisymmetric}{\left(\begin{array}{cc}
\sigma & 0 \\ 
 0 & \sigma \\
1-\sigma & 1-\sigma
\end{array} \right)}

\newcommand{\Msemisymmetric}{\left(\begin{array}{cc}
\frac{1 - 2 \sigma + 2 \sigma^2}{1- 3\sigma + 5 \sigma^2 - 3 \sigma^3} & \frac{-\sigma^2}{1- 3\sigma + 5 \sigma^2 - 3 \sigma^3} \\ 
\frac{-\sigma^2}{1- 3\sigma + 5 \sigma^2 - 3 \sigma^3} & \frac{1 - 2 \sigma + 2 \sigma^2}{1- 3\sigma + 5 \sigma^2 - 3 \sigma^3} \\
\frac{\sigma }{1 - 2 \sigma + 3 \sigma^2} & \frac{\sigma }{1 - 2 \sigma + 3 \sigma^2}
\end{array} \right)}

\newcommand{\Tsemi}{\left(\begin{array}{cc}
\sigma_{-1} & 0 \\ 
 0 & \sigma_{1} \\
1 - \sigma_{-1} & 1 - \sigma_{1}
\end{array} \right)}

\begin{document}

\maketitle

\begin{abstract}
In supervised learning one wishes to identify a pattern present in a joint distribution $P$, of instances, label pairs, by providing a function $f$ from instances to labels that has low risk $\EE_{P}\ell(y,f(x))$. To do so, the learner is given access to $n$ iid samples drawn from $P$. In many real world problems clean samples are not available. Rather, the learner is given access to samples from a corrupted distribution $\tilde{P}$ from which to learn, while the goal of predicting the clean pattern remains. There are many different types of corruption one can consider, and as of yet there is no general means to compare the relative ease of learning under these different corruption processes. In this paper we develop a general framework for tackling such problems as well as introducing upper and lower bounds on the risk for learning in the presence of corruption. Our ultimate goal is to be able to make informed economic decisions in regards to the acquisition of data sets. For a certain subclass of corruption processes (those that are \emph{reconstructible}) we achieve this goal in a particular sense. Our lower bounds are in terms of the coefficient of ergodicity \cite{Dobrushin1956}, a simple to calculate property of stochastic matrices. Our upper bounds proceed via a generalization of the method of unbiased estimators appearing in \cite{Natarajan2013} and implicit in the earlier work \cite{Kearns1998}.
\end{abstract}

% ====================================================================================================================================================================================================================

\section{Introduction}

The goal of supervised learning is to find a function in some hypothesis class that predicts a relationship between instances and labels. Such a function should have low average loss according to the true distribution of instances and labels, $P$. The learner is not given direct access to $P$, but rather a training set comprising $n$ iid samples from $P$. There are many algorithms for solving this problem (for example empirical risk minimization) and this problem is well understood.
\\
\\
There are many other \emph{types} of data one could learn from. For example in semi-supervised learning \cite{Chapelle2010} the learner is given $n$ instance label pairs and $m$ instances devoid of labels. In learning with noisy labels \cite{Angluin1988,Kearns1998,Natarajan2013}, the learner observes instance label pairs where the observed labels have been corrupted by some noise process. There are many other variants including, but not limited to, learning with label proportions \cite{Quadrianto2009}, learning with partial labels \cite{Cour2011}, multiple instance learning \cite{Maron1998} as well as combinations of the above. 
\\
\\
What is currently lacking is a general theory of learning from corrupted data, as well as means to \emph{compare} the relative usefulness of different data types. Such a theory is required if one wishes to make informed economic decisions on which data sets to acquire. For example, are $n$ clean datum better or worse than $n_1$ noisy labels and $n_2$ partial labels?
\\
\\
To answer this question we first place the problem of corrupted learning into the abstract language of statistical decision theory. We then develop general lower and upper bounds on the risk relative to the amount of corruption of the clean data. Finally we show examples of problems that fit into this abstract framework.
\\
\\
The main contributions of this paper are: 

\begin{itemize}
	\item Novel, general means to construct methods for learning from corrupted data based on a generalization of the method of unbiased estimators presented in \cite{Natarajan2013} and implicit in the earlier work \cite{Kearns1998} (theorems \ref{MUBE} and \ref{Corrupted PAC Bayes})
	\item Novel lower bounds on the risk of corrupted learning (theorem \ref{Relative Lower Bound}).
	\item Means to understand \emph{compositions} of corruptions (lemmas \ref{Alpha Composition} and \ref{Loss norm Composition}).
	\item Upper and lower bounds on the risk of learning from combinations of corrupted data (theorems \ref{Collection of Corrupted PAC Bayes} and \ref{Collection of Corrupted Lower Bound}).
	\item Analyses of the tightness of the above bounds.
\end{itemize}
In doing so we provide answers to our central question of how to rank different types of corrupted data, through the utilization of our upper or lower bounds. While not the complete story for \emph{all} problems, the contributions outlined above make progress toward the final goal of being able to make informed economic decisions regarding the acquisition of data sets. All proofs omitted in the main text appear in the appendix.

% ====================================================================================================================================================================================================================

\section{The Decision Theoretic Framework}

Decision theory deals with the general problem of decision making under uncertainty. One starts with a set $\Theta$ of possible true hypotheses (only one of which is actually true) as well as set $A$ of actions available to the decision maker. Prior to acting, the decision maker performs an experiment, the outcome of which is assumed to be related to the true hypothesis, and observes $z$ in an observation space $\obs$. Ultimately the decision maker makes act $a$ and incurs loss $L(\theta,a)$, with $\theta$ the unknown true hypothesis. We model the relationships between unknowns and the results of experiments with \emph{Markov kernels} \cite{Torgersen1991,Cam2011,Morse1966,Chentsov1982}. The abstract development that follows is necessary in order to place a wide range of corruption processes into a single framework so that they may be compared. 
\subsection{Markov Kernels}\label{sec:markov-kernels}
As much of our focus will be on noise on the \emph{labels} and not on the \emph{instances}, henceforth we will assume we are only working with finite sets.
\\
Denote by $\PP(X)$ the set of probability distributions on a set $X$. Define a \emph{Markov kernel} from a set $X$ to a set $Y$ (denoted by $X \rightsquigarrow Y$) to be a function $T : X \rightarrow \PP(Y)$. Denote the set of all Markov kernels from $X$ to $Y$ by $M(X,Y)$. Every function $f : X \rightarrow Y$ defines a Markov kernel $T : X \rightsquigarrow Y$ with $T(X) = \delta_{f(x)}$, a point mass on $f(x)$. Given two Markov kernels $T_1 : X \rightsquigarrow Y $ and $T_2 : Y \rightsquigarrow Z$ we can \emph{compose} them to form $T_2  T_1 : X \rightsquigarrow Z$  by taking 
$$
\EE_{T_2 T_1(x)} f =  \EE_{ y \dist T_1(x)} \EE_{z \dist T_2(y)} f(z)
$$ 
for all $f: Z \rightarrow \RR$. One can also combine Markov kernels in parallel. If $P \in \PP(X)$ and $Q \in \PP(X)$, denote the product distribution by $P \otimes Q$. If $T_i : X \rightsquigarrow Y$, $ i \in [1;n]$, are Markov kernels then $\otimes_{i=1}^n T_i : X^n \rightsquigarrow Y^n$ with $\otimes_{i=1}^n T_i(x^n) = T_1(x_1) \otimes \dots \otimes T_n(x_n)$. By restricting ourselves to finite sets, distributions can be represented by vectors, Markov kernels by column stochastic matrices (positive matrices with column sum 1) and composition by matrix multiplication. An \emph{experiment} on $\Theta$ is any Markov kernel with domain $\Theta$ and a \emph{learning algorithm} $\alg$ is any Markov kernel with co-domain $A$. Finally, from any experiment $e : \Theta \rightsquigarrow \obs$ we define the \emph{replicated experiment} $e_n : \Theta \rightsquigarrow \obs^n ,\ n\in \{1,2,\dots\}$, with $e_n(\theta) = e(\theta)^n$ the $n$-fold product of $e(\theta)$. 

\subsection{Loss and Risk}
One assesses the \emph{consequence} of actions through a \emph{loss} $L : \Theta \times A \rightarrow \RR$. It is sometimes useful to work with losses in curried form. From any loss $L$ and action $a \in A$, define $L_a \in \RR^ \Theta$ with  $L_a(\theta) = L(\theta,a)$. We measure the size of a loss function by its supremum norm $\lVert L \rVert_\infty = \sup_{\theta,a} |L(\theta,a)|$. If $P\in\PP(\Theta)$ and $Q \in \PP(A)$ we overload our notation with $L(P,Q) = \EE_{\theta \dist P} \EE_{a \dist Q} L(\theta,a)$.
\\
\\
Normally, we are not interested in the absolute loss of an action, rather its loss relative to the best action, defined formally as the \emph{regret} $\relloss(\theta,a) = L(\theta,a) - \inf_{a'} L(\theta,a')$. We measure the performance of an algorithm $\alg$ by the \emph{risk}
$$
\riskL{L}(e, \theta, \alg) = \EE_{z \dist e(\theta)} \EE_{a \dist \alg(z)} \relloss(\theta,a).
$$
For the sake of comparison by a single number either the max risk or the average risk with respect to a distribution $P_{\Theta} \in \PP(\obs)$ can be used. We define a \emph{learning problem} to be a pair $(L, e)$ with $L : \Theta \times A \rightarrow \RR$ a loss and $e : \Theta \rightsquigarrow \obs$ an experiment. We measure the difficulty of a learning problem by the \emph{minimax risk}
$$
\Rbar{L}(e) =  \inf_{\alg} \sup_{\theta} \riskL{L}(e, \theta, \alg).
$$
Normally we are not concerned with the quality of a learning algorithm for observation of a single $z\in \obs$. Rather we wish to know the rate at which the risk decreases as the number of replications of the experiment grows. Hence the prime quantity of interest is $\Rbar{L}(e_n)$.

\subsection{Statistics vs Machine Learning}

While the ideas of the previous subsections originated in theoretical statistics \cite{Torgersen1991,Cam2011,Blackwell:1954,Ferguson1967} they can be readily applied to machine learning problems. The main distinction is that statistics focuses on \emph{parametric families} and loss functions of type $L : \Theta \times \Theta \rightarrow \RR$. The goal is to accurately \emph{reconstruct parameters}. In machine learning one is interested in \emph{predicting the observations} of the experiment well. There the focus is on problems with $\Theta = \PP(\obs)$ and loss functions of the form $L(\theta, a) = \EE_{z \dist P_{\theta}} \ell(z,a)$, where $\ell : \obs \times A \rightarrow \RR$ measures how well $a$ predicts the observation $z$. Our focus is on problems of the second sort, however abstractly there is no real difference. Both are just different learning problems. When clear we use $\ell(P,a)$ and $L(P,a)$ interchangeably.

\subsubsection{Supervised Learning}

In Table 1 we explain the mapping of supervised learning into our abstract language. We focus on the problem of conditional probability estimation of which learning a binary classifier is a special case. Letting $X$ be the instance space and $Y$ the label space we have

\begin{center}
\captionof{table}{Supervised Learning}
\begin{tabular}{l||l}
	   Unknowns $\Theta$  & Distributions of instance, label pairs, $\PP(X\times Y)$ \\ 
\hline Observation Space $\obs$ & $n$ instance label pairs $(X \times Y)^n$.  \\ 
\hline Action Space $A$ & Function class $\mathcal{F} \subseteq \PP(Y)^X$ \\ 
\hline Experiment $e$ & Maps each $P \in \PP(X\times Y)$ to itself \\ 
\hline Loss $L$ & $L(\theta,f) = \EE_{(x,y)\dist P_\theta} \ell(y, f(x))$ \\ 

\end{tabular}
\end{center}
We have  
$$
\riskL{L} (e_n, P, \alg) = \EE_{S \dist P^n} \EE_{f\dist \alg(S)} \EE_{(x,y) \dist P} \ell(y,f(x)) - \inf_{f \in \mathcal{F}} \EE_{(x,y) \dist P} \ell(y,f(x))
$$
a standard object of study in learning theory \cite{Bousquet2004}.

\subsection{Corrupted Learning}

In corrupted learning, rather than observing $z \in \obs$, one observes a corrupted $\tilde{z}$ in a different observation space $\tilde{\obs}$. We model the corruption process through a Markov kernel $T : \obs \rightsquigarrow \tilde{\obs}$ and define a corrupted learning problem to be the triple $(L,e,T)$. For convenience we define the corrupted experiment $\tilde{e} = T e$. Ideally we wish to compare $\Rbar{L}(\tilde{e}_n) $ with $\Rbar{L}(e_n)$. By general forms of the information processing theorem \cite{Reid2009b,Garca-Garca} $\Rbar{L}(\tilde{e}_n) \geq \Rbar{L}(e_n)$, however this does not allow one to \emph{rank} the utility of \emph{different} $T$.
\\
\\
Even after many years of directed research, in general we can not compute $\Rbar{L}(e_n)$ exactly, let alone $\Rbar{L}(\tilde{e}_n)$ for general corruptions. Consequently our effort for the remaining turns to upper and lower bounds of $\Rbar{L}(\tilde{e}_n)$.

% ====================================================================================================================================================================================================================

\section{Upper Bounds for Corrupted Learning}\label{sec:upper-bounds-for-corrupted-learning}

When convenient we use the shorthand $T(P) = \tilde{P}$. \cite{Natarajan2013} introduced a method of learning classifiers from data subjected to label noise, termed the \emph{method of unbiased estimators}. Here we show that this method can be generalized to other corruptions. Firstly, $\PP(\obs) \subseteq (\RR^\obs)^*$, the dual space of $\RR^\obs$. We use the notation $\langle P, f \rangle = \EE_{z \dist P} f(z)$. From any markov kernel $T : \obs \rightsquigarrow \tilde{\obs}$, we obtain a linear map $T : (\RR^\obs)^* \rightarrow (\RR^{\tilde{\obs}} )^*$ with
$$
\ip{T(\alpha)}{\tilde{f}} = \ip{\alpha}{T^*(\tilde{f})} ,\ \forall \tilde{f} \in \RR^{\tilde{\obs}}
$$
where $T^*(\tilde{f})(z) = \EE_{\tilde{z} \dist T(z)} \tilde{f}(\tilde{z})$ is the \emph{pullback} of $\tilde{f}$ by $T$. In terms of matrices $T^*$ is the transpose or \emph{adjoint} of T.

\begin{definition}
A Markov kernel $T : \obs \rightsquigarrow \tilde{\obs}$ is \emph{reconstructible} if $T$ has a left inverse, there exists a linear map $R : (\RR^{\tilde{\obs}})^* \rightarrow (\RR^\obs)^*$ such that $R T = I$.
\end{definition}
Intuitively, $T$ is reconstructible if there is some transformation that ``undoes" the effects of $T$. In general $R$ is not a Markov kernel. Many forms of corrupted learning are reconstructible, including semi-supervised learning, learning with label noise and learning with partial labels for all but a few pathological cases. The reader is directed to \ref{Examples} for worked examples.
\\
\\
We call a left inverse of $T$ a \emph{reconstruction}. For concreteness, one can always take 
$$
R = (T^* T)^{-1} T^*
$$
the Moore-Penrose pseudo inverse of $T$. Reconstructible Markov kernels are exactly those where we can \emph{transfer} a loss function from the clean distribution to the corrupted distribution. We have by properties of adjoints
$$
\ip{P}{f} = \ip{R T (P)}{f} = \ip{T(P)}{R^*(f)}.
$$
In words, to take expectations of $f$ with samples from $\tilde{P}$ we use the corruption corrected $\tilde{f} = R^*(f)$. 

\begin{theorem}[Corruption Corrected Loss]\label{MUBE}
For all reconstructible $T : \obs \rightsquigarrow \tilde{\obs}$, loss functions $\ell : \obs \times A \rightarrow \RR$ and reconstructions $R$ define the \emph{corruption corrected} loss $\tilde{\ell} : \tilde{\obs} \times A \rightarrow \RR$, with  $\tilde{\ell}_a = R^*\ell_a$. Then for all distributions $P \in \PP(\obs)$, $\ell(P,a) = \tilde{\ell}(\tilde{P},a)$.

\end{theorem}
We direct the reader to \ref{Examples} for some examples of $\tilde{\ell}$ for different corruptions. Minimizing $\tilde{\ell}$ on a sample $\tilde{S} \dist \tilde{P}$ provides means to learn from corrupted data. Let $\ell(S,a) = \frac{1}{|S|} \sum_{z\in S} \ell(z,a)$, the average loss on the sample. By an application of the PAC Bayes bound (\cite{McAllester1998,Zhang2006,Catoni2007}) one has for all algorithms $\alg : \tilde{\obs}^n \rightsquigarrow A$, priors $\pi \in \PP(A)$ and distributions $P \in \PP(\obs)$
$$
\EE_{\tilde{S}\dist \tilde{P}^n} \tilde{\ell}(\tilde{P},\alg(\tilde{S})) \leq \EE_{\tilde{S}\dist \tilde{P}^n} \tilde{\ell}(\tilde{S},\alg(\tilde{S}))  +  \lVert \tilde{\ell}\rVert_\infty \sqrt{\frac{2 \EE_{\tilde{S}\dist \tilde{P}^n} D_{KL}(\alg(\tilde{S}),\pi)}{n}}.
$$
This bound yields the following theorem.

\begin{theorem}\label{Corrupted PAC Bayes}
For all reconstructible Markov kernels $T : \obs \rightarrow \tilde{\obs}$, algorithms $\alg : \tilde{\obs}^n \rightsquigarrow A$, priors $\pi \in \PP(A)$, distributions $P \in \PP(\obs)$ and bounded loss functions $\ell$
$$
\EE_{\tilde{S}\dist \tilde{P}^n} \ell(P,\alg(\tilde{S})) \leq \EE_{\tilde{S}\dist \tilde{P}^n}\tilde{\ell}(\tilde{S},\alg(\tilde{S}))  +  \lVert \tilde{\ell}\rVert_\infty \sqrt{\frac{2 \EE_{\tilde{S}\dist \tilde{P}^n} D_{KL}(\alg(\tilde{S}),\pi)}{n}}.
$$
\end{theorem}
A similar result also holds with high probability on draws from $\tilde{P}^n$.  If $\alg$ is Empirical Risk Minimization (ERM), $A$ is finite and $\pi$ uniform on $A$ the above analysis yields convergence to the optimum $a\in A$ as $\frac{ \lVert \tilde{\ell}\rVert_\infty}{\sqrt{n}}$ for learning with corrupted data versus $\frac{\lVert \ell\rVert_\infty}{\sqrt{n}}$ for learning with clean data. Therefore, the ratio $\frac{\lVert \tilde{\ell}\rVert_\infty}{\lVert \ell\rVert_\infty}$ measures the relative difficulty of corrupted versus clean learning.

\subsection{Upper Bounds for Combinations of Corrupted Data}\label{sec:upper-bounds-for-combinations-of-corrupted-data}

Recall that our final goal is to be able to make informed economic decisions in regarding the acquisition of data sets. As such, we wish to quantify the utility of a data set comprising different corrupted data. For example in learning with noisy labels out of $n$ datum, there could be $n_1$ clean, $n_2$ slightly noisy and $n_3$ very noisy samples and so on. More generally we assume access to a corrupted sample $\tilde{S}$, made up of $k$ different types of corrupted data, with $\tilde{S}_i \dist \tilde{P}^{n_i}$.

\begin{theorem}\label{Collection of Corrupted PAC Bayes}
Let $T_i : \obs \rightsquigarrow \tilde{\obs}_i$ be a collection of $k$ reconstructible Markov kernels. Let $\tilde{Q} = \otimes_{i=i}^k \tilde{P}_i^{n_i}$ and $\tilde{\obs} = \times_{i=1}^k \tilde{\obs}_i^{n_i}$, $n = \sum_{i=1}^k n_i$ and $r_i = \frac{n_i}{n}$. Then for all algorithms $\alg : \tilde{\obs} \rightsquigarrow A$, priors $\pi \in \PP(A)$, distributions $P \in \PP(\obs)$ and bounded loss functions $\ell$
$$
\EE_{\tilde{S} \dist \tilde{Q}} \ell(P,\alg(\tilde{S})) \leq \EE_{\tilde{S}\dist \tilde{Q}} \sum_{i=1}^k r_i \tilde{\ell}_i(\tilde{S}_i,\alg(\tilde{S})) +  K \sqrt{\frac{2 \EE_{\tilde{S}\dist \tilde{Q}} D_{KL}(\alg(\tilde{S}),\pi)}{n}}.
$$
where $K = \sqrt{\sum\limits_{i=1}^{k}r_i \lVert \tilde{\ell}_i\rVert_{\infty}^2}$.

\end{theorem}
A similar result also holds with high probability on draws from $Q$. Theorem \ref{Collection of Corrupted PAC Bayes} is a generalization of the final bound appearing in \cite{Crammer2005} that only pertains to symmetric label noise and binary classification. Theorem \ref{Collection of Corrupted PAC Bayes} suggest the following means of choosing data sets. Let $c_i$ be the cost of acquiring data corrupted by $T_i$ and $C$ the maximum total cost. First, choose data from the $T_i$ with lowest $c_i \lVert \tilde{\ell}_i\rVert_{\infty}^2$ until picking more violates the budget constraint. Then choose data from the second lowest and so on.
%================================================================================================================================================================================================================================================

\section{Lower Bounds for Corrupted Learning}

Thus far we have developed upper bounds for ERM style algorithms. In particular we have found that reconstructible corruption does not effect the \emph{rate} at which learning occurs, it only effects constants in the upper bound. Can we do better? Are these constants \emph{tight}? To answer this question we develop lower bounds for corrupted learning.
\\
\\
Here we review Le Cam's method \cite{Cam2011} a powerful technique for generating lower bounds for learning problems that very often gives the correct rate and dependence on constants (including being able to reproduce the standard VC dimension lower bounds for classification presented in \cite{Massart2006}). In recent times it has been used to establish lower bounds for: differentially private learning \cite{Duchi2013}, learning in a distributed set up \cite{Zhang2013}, function evaluations required in convex optimization \cite{Agarwal2009} as well as generic lower bounds in statistical estimation problems \cite{Yang1999}. We show how this method can be extended using the strong data processing theorem \cite{Boyen1998,Cohen1998} to provide a general tool for lower bounding corrupted learning problems. 
\subsection{Le Cam's Method and Minimax Lower Bounds}

Le Cam's method proceeds by reducing a general learning problem to an easier binary classification problem, before relating the best possible performance on this classification problem to the minimax risk. Define the \emph{separation} $\rho : \Theta \times \Theta \rightarrow \RR$, $\rho(\theta_1,\theta_2) = \inf_a \relloss(\theta_1,a) + \relloss(\theta_2,a)$. The separation measures how hard it is to act well against both $\theta_1$ and $\theta_2$ simultaneously. We have the following (see section \ref{sec:le-cam's-method-and-minimax-lower-bounds PROOF} for a more detailed treatment).

\begin{lemma}\label{Le Cam Lemma}
For all experiments $e$, loss functions $L$ and $\theta_1, \theta_2 \in \Theta$
$$
\Rbar{L} (e) \geq  \rho(\theta_1,\theta_2) \left(\frac{1}{4} - \frac{1}{4} V(e(\theta_1), e(\theta_2)) \right).
$$
where $V$ is the variational divergence.
\end{lemma}
This lower bound is a trade off between distances measured by $\rho$ and statistical distances measured by the variational divergence. A learning problem is easy if proximity in variational divergence of $e(\theta_1)$ and $e(\theta_2)$ (hard to distinguish $\theta_1$ and $\theta_2$ statistically) implies proximity of $\theta_1$ and $\theta_2$ in $\rho$ (hard to distinguish $\theta_1$ and $\theta_2$ with actions).
\\
\\
If there exists $\theta_1, \theta_2$ with $e(\theta_1) = e(\theta_2)$ and $\rho(\theta_1,\theta_2) > 0$ we instantly get that the minimax regret must be positive. For corrupted experiments, if $T$ is not reconstructible it may be the case that  $T e(\theta_1) = T e(\theta_2)$ for some $\theta_1, \theta_2$. Hence we assume that $T$ is reconstructible.

\subsubsection{Replication and Rates}\label{Duchi Method}

We wish to lower bound how the risk decreases as $n$ grows. When working with replicated experiments it can be advantageous to work with an $f$-divergence (see section \ref{sec:a-generic-strong-data-processing-theorem.}) different to variational divergence and to invoke a generalized Pinkser inequality \cite{Reid2009b}. Common choices in theoretical statistics are the Hellinger and alpha divergences \cite{Guntuboyina2011} as well as the KL divergence \cite{Duchi2013}. Here we use the variational divergence and the following lemma.

\begin{lemma}\label{Variational Divergence for product distribitions}
For all collections of distributions $P_i, Q_i \in \PP(\obs_i)$, $i \in [1 ; k]$
$$
V(\otimes_{i=1}^k P_i, \otimes_{i=1}^k Q_i) \leq \sum_{i=1}^k V(P_i,Q_i)
$$
\end{lemma}
Here we make use of the specific case where $P_i = P$ and $Q_i = Q$ for all $i$.

\begin{lemma}\label{Replicated LeCam}
For all experiments $e$, loss functions $L$, $\theta_1, \theta_2 \in \Theta$ and $n$
$$
\Rbar{L}(e_n) \geq \rho(\theta_1,\theta_2) \left( \frac{1}{4} - \frac{n}{4} V(e(\theta_1),e(\theta_2)) \right).
$$
\end{lemma}
To use lemma \ref{Replicated LeCam}, one defines $\theta_1 = \phi_1(n)$ and $\theta_2 = \phi_2(n)$ for $n \in [0,\infty)$, with the property 
$$
\frac{1}{4} - \frac{n}{4} V(e(\theta_1),e(\theta_2)) \geq \frac{1}{8}
$$ 
or equivalently $V(e(\theta_1), e(\theta_2)) \leq \frac{1}{2 n}$. This yields a lower bound of 
\begin{equation*}
\Rbar{L}( e_n)  \geq \frac{1}{8} \rho(\phi_1(n), \phi_2(n) ).
\end{equation*}
To obtain \emph{tight} lower bounds, $\phi$ needs to be designed in a problem dependent fashion. However, as our goal here is to reason \emph{relatively} we assume that $\phi$ is given.

\subsubsection{Other Methods for Obtaining Minimax Lower Bounds}\label{Other Methods}

There are many other techniques for lower bounds in terms of functions of pairwise $KL$ divergences \cite{Yu1997} (for example Assouad's method) as well as functions of pairwise f-divergences \cite{Guntuboyina2011}. While such methods are often required to get tighter lower bounds, all of what follows can be applied to these more intricate lower bounding techniques. Therefore, for the sake of conceptual clarity, we proceed with Le Cam's method.

\subsection{Measuring the Amount of Corruption}

Rather than the experiment $e$, in corrupted learning we work with the corrupted experiment $\tilde{e}$. By the information processing theorem for $f$-divergences \cite{Reid2009b}, states that 
$$
V(T(P),T(Q)) \leq V(P,Q) ,\ \forall P, Q
$$ 
Thus any lower bound achieved by Le Cam's method for $e$ can be directly transferred to one for $\tilde{e}$. This is just a manifestation of theorems presented in \cite{Reid2009b,Garca-Garca} and alluded to in section 2.4. However, this provides us with no means to rank different $T$. For some $T$, the information processing theorem can be \emph{strengthened}, in the sense that one can find $\alpha(T) < 1$ such that
$$
\forall P,Q,\ V(T(P),T(Q)) \leq \alpha(T)V(P,Q).
$$ 
The coefficient $\alpha(T)$ provides a means to measure the amount of corruption present in $T$. For example if $T$ is constant and maps all $P$ to the same distribution, then $\alpha(T)=0$. If $T$ is an invertible function, then $\alpha(T) = 1$. Together with lemma \ref{Replicated LeCam} this strong information processing theorem \cite{Cohen1998} leads to meaningful lower bounds that allow the comparison of different corrupted experiments.

\subsection{A Generic Strong Data Processing Theorem.}\label{sec:a-generic-strong-data-processing-theorem.}
Following \cite{Cohen1998}, we present a strong data processing theorem that works for all $f$-divergences. 
\begin{definition}
Let $X$ be a set and $f : \RR_+ \rightarrow \RR$ a convex function with $f(1) = 0$. For all distributions $P,Q \in \PP(X)$ the \emph{$f$-divergence} between $P$ and $Q$ is 
$$
D_f(P,Q) = \int_X f\left(\frac{dQ}{dP}\right) dP.
$$
\end{definition}
Both the variational and KL divergence are examples of $f$ divergences. For fixed $T$ we seek an $\alpha(T)$ such that 
$$
D_f(T(P),T(Q)) \leq \alpha(T) D_f(P,Q) \ \forall P, Q, f. 
$$
To do so we first relate the amount $T$ \emph{contracts} $P$ and $Q$ to a certain deconstruction for Markov kernels before proving when such a deconstruction can occur.

\begin{lemma}\label{Deconstruction and KL Lemma}
For all Markov kernels $T : X \rightsquigarrow Y$ and distributions $P,Q \in \PP(X)$, if there exists $F,G \in M(X,Y)$ and $\lambda \in [0,1]$ such that $T = \lambda F + (1-\lambda) G$ with $F(P) = F(Q)$ then $\Df(T(P),T(Q)) \leq (1-\lambda) \Df(P,Q)$.
\end{lemma}
Hence the amount $T$ contracts $P$ and $Q$ is related to the amount of $T$ that fixes $P$ and $Q$. We seek the largest $\lambda$ such that a decomposition $T = \lambda F + (1-\lambda) G$ is always possible, no matter what pair of distributions $F$ is required to fix.

\begin{lemma}\label{Existence of Decontruction Lemma}
For all Markov kernels $T : X \rightsquigarrow Y$ define $\lambda(T) = \min_{i,j} \sum_k \min(T_{k,i}, T_{k,j})$. Then $\lambda \leq \lambda(T)$ if and only if for all pairs of distributions $P,Q$ there exists a decomposition $$T = \lambda F + (1-\lambda) G$$ with $F,G\in M(X,Y)$ and $F(P) = F(Q)$. 
\end{lemma}

\begin{theorem}[Strong Data Processing]\label{Strong Data Processing}
For all Markov kernels $T : X \rightsquigarrow Y$ define $\alpha(T) = 1 -\lambda(T)$. Then for all $P,Q,f$, 
$$
\Df(T(P),T(Q))\leq \alpha(T) \Df(P,Q).
$$
\end{theorem}
The proof is a simple application of lemma \ref{Deconstruction and KL Lemma} and lemma \ref{Existence of Decontruction Lemma}. It is easy to see that $0 \leq \alpha(T) \leq 1$. Furthermore $\alpha(T) = 0$ if and only if all of the columns of $T$ are the same. While this $\alpha$ may not be the tightest for a given $f$, it is \emph{generic} and as such can be applied in all lower bounding methods mentioned previously.

\subsection{Relating $\alpha$ to Variational Divergence}\label{sec:Variational Alpha}

It can be shown \cite{Cohen1998} that $\alpha(T) = \max_{x_1, x_2} V(T(x_1), T(x_2)) = \frac{1}{2} \max_{i,j} \sum_k|T_{ki} - T_{kj}|$, the maximum $L1$ distance between the columns of $A$ \cite{Reid2009b}. Furthermore
$$
\alpha(T) = \sup_{P, Q \in \PP(X)} \frac{V(T(P), T(Q))}{V(P, Q) } = \sup_{v \in S} \frac{\lVert T(v) \rVert_1}{\lVert v \rVert_1}
$$
where $S = \{v : \sum v_i = 0, v \neq 0\}$. Hence $\alpha(T)$ is the operator 1-norm of T when restricted to $S$. The above also shows that $\alpha(T)$ provides the tightest strong data processing theorem possible when using variational divergence, and hence it gives the tightest generic strong data processing theorem. We also have the following compositional property of $\alpha$.
\begin{lemma}\label{Alpha Composition}
For all Markov kernels $T_1 : X \rightsquigarrow Y$ and $T_2 : Y \rightsquigarrow Z$, 
$$
\alpha(T_2 T_1) \leq \alpha(T_2) \alpha(T_1)  \leq \min(\alpha(T_2), \alpha(T_1)).
$$

\end{lemma}
Hence $T_2 T_1$ is at least as corrupt as either of the $T_i$.
\\
\\
The first use of $\alpha(T)$ occurs in the work of \cite{Dobrushin1956} where it is called the coefficient of ergodicity and is used (much like in \cite{Boyen1998}) to prove rates of convergence of Markov chains to their stationary distribution.

\subsection{Lower bounds Relative to the Amount of Corruption}

\begin{lemma}
For all experiments $e$, loss functions $L$, $\theta_1, \theta_2 \in \Theta$, $n$ and corruptions $T : \obs \rightsquigarrow \tilde{\obs}$
$$
\Rbar{L}(\tilde{e}_n) \geq \rho(\theta_1,\theta_2) \left( \frac{1}{4} - \frac{\alpha(T) n}{4} V(e(\theta_1),e(\theta_2)) \right).
$$
\end{lemma}
The proof is a simple application of lemma \ref{Replicated LeCam} and the strong data processing. Suppose we have proceeded as in section \ref{Duchi Method}, defining $\theta_1 = \phi_1(n)$ and $\theta_2 = \phi_2(n)$ with $V(e(\theta_1), e(\theta_2) ) \leq \frac{1}{2 t}$. Letting $\tilde{\theta}_1 = \phi_1(\alpha(T) n)$ and $\tilde{\theta}_2 = \phi_2(\alpha(T) n)$ gives $V(e(\tilde{\theta}_1), e(\tilde{\theta}_2)) \leq \frac{1}{2 \alpha(T) n}$. Furthermore
$$
\Rbar{L}(\tilde{e}_n) \geq \frac{1}{8} \rho(\phi_1(\alpha(T) n), \phi_2(\alpha(T) n) ).
$$
In words, if ever Le Cam's method gives a lower bound of $f(n)$ for repetitions of the clean experiment, we obtain a lower bound of $f(\alpha(T) n)$ for repetitions of the corrupted experiment. Hence the \emph{rate} is unaffected, only the constants. However, a penalty of factor $\alpha(T)$ is unavoidable no matter what learning algorithm is used, suggesting that $\alpha(T)$ is a valid way of measuring the amount of corruption. We summarize the results of this section in the following theorem.

\begin{theorem}\label{Relative Lower Bound}
For all corruptions $T : \obs \rightsquigarrow \tilde{\obs}$ and experiments $e : \Theta \rightsquigarrow \obs$, if Le Cam's method yields a lower bound $\Rbar{L}(e_n) \geq f(n)$ then $\Rbar{L}(\tilde{e}_n) \geq f(\alpha(T) n).$
\end{theorem}
In particular if one has a lower bound of $\frac{C}{\sqrt{n}}$ for the clean problem, as is usual for many machine learning problems, theorem \ref{Relative Lower Bound} yields a lower bound of $\frac{C}{\sqrt{\alpha(T) n}}$ for the corrupted problem.

\subsection{Lower Bounds for Combinations of Corrupted Data}\label{sec:lower-bounds-for-combinations-of-corrupted-data}

As in section \ref{sec:upper-bounds-for-combinations-of-corrupted-data} we present lower bounds for combinations of corrupted data. For example in learning with noisy labels out of $n$ datum, there could be $n_1$ clean, $n_2$ slightly noisy and $n_3$ very noisy samples and so on. 

\begin{theorem}\label{Collection of Corrupted Lower Bound}
Let $T_i : \obs \rightsquigarrow \tilde{\obs}_i$, $i \in [1;k]$, be reconstructible Markov kernels. Let $T = \otimes_{i=i}^k T_i^{n_i}$ with $n = \sum_{i=i}^k n_k$. If Le Cam's method yields a lower bound $\Rbar{L}(e_n) \geq f(n)$ then\\ $\Rbar{L}(T e_n) \geq f(K)$ where $K = \left( \sum\limits_{i=1}^{k}\alpha(T_i) n_i \right)$.
\end{theorem}
As in section \ref{sec:upper-bounds-for-combinations-of-corrupted-data} this bound suggest means of choosing data sets, via the following integer program 
$$
\max \sum\limits_{i=1}^{k}\alpha(T_i) n_i  \ \ \text{subject to} \sum\limits_{i=1}^k c_i n_i \leq C
$$
where $c_i$ is the cost of acquiring data corrupted by $T_i$ and $C$ is the maximum total cost. This is exactly the unbounded knapsack problem \cite{Dantzig1957} which admits the following near optimal greedy algorithm. First, choose data from the $T_i$ with highest $\frac{\alpha(T_i)}{c_i}$ until picking more violates the constraints. Then pick from the second highest and so on. 

%====================================================================================================================================================================================================================================================

\section{Measuring the Tightness of the Upper Bounds and Lower Bounds}

In the previous sections we have shown upper bounds that depend on $\lVert \tilde{\ell}\rVert_\infty$ as well as lower bounds that depend on $\alpha(T)$. Recall from theorem that \ref{MUBE} $\tilde{\ell}_a = R^* \ell_a$, as such the worst case ratio $\frac{\lVert \tilde{\ell}\rVert_\infty}{\lVert \ell\rVert_\infty}$ is determined by the \emph{operator norm} of $R^*$. For a linear map $R : \RR^X \rightarrow \RR^Y$ define
\begin{align*}
\lVert R \rVert_{1} := \sup_{v \in \RR^X} \frac{\lVert R v \rVert_{1}}{\lVert v\rVert_{1}} & \ , \lVert R \rVert_{\infty} := \sup_{v \in \RR^X} \frac{\lVert R v \rVert_{\infty}}{\lVert v\rVert_{\infty}}
\end{align*}
which are two operator norms of $R$. They are equal to the maximum absolute column and row sum of $R$ respectively \cite{Bernstein2009}. Hence $\lVert R \rVert_{1} = \lVert R^* \rVert_{\infty}$. 

\begin{lemma}\label{noisy loss norm}
For all losses $\ell$, $T : \obs \rightarrow \tilde{\obs}$ and reconstructions $R$, $\frac{\lVert \tilde{\ell}\rVert_\infty}{\lVert \ell\rVert_\infty} \leq \lVert R^* \rVert_{\infty}$.
\end{lemma}

\begin{lemma}\label{adjoint inverses}
If $T : X \rightsquigarrow Y$ is reconstructible, with reconstruction $R$, then 
$$
\frac{1}{\alpha(T)} \leq 1 / \left( \inf_{u \in \RR^X} \frac{\lVert T u\rVert_{1}}{\lVert u \rVert_{1}} \right) \leq \lVert R^* \rVert_{\infty}. 
$$
\end{lemma}
The intuition here is if $T$ contracts a particular $v \in \RR^X$ greatly, which would occur if 
$$
\inf_{P, Q \in \PP(X)} \frac{\lVert T(P - Q)  \rVert_1}{\lVert P - Q \rVert_1} 
$$
was small (here $v = P - Q$), then $R^*$ could greatly increase the norm of a loss $\ell$. However, it need not increase the norm of the particular loss of interest. Note that for lower bounds we look at the \emph{best} case separation of columns of $T$, for upper bounds we essentially use the \emph{worst}. We also get the following compositional theorem.
\begin{lemma}\label{Loss norm Composition}
If $T_1 : X \rightsquigarrow Y$ and $T_2 : Y \rightsquigarrow Z$ are reconstructible, with reconstructions $R_1$ and $R_2$ then $T_2 T_1$ is reconstructible with reconstruction $R_1 R_2$. Furthermore $\frac{1}{\alpha(T_1)\alpha(T_2)} \leq \lVert R_1 R_2\rVert_{1} \leq \lVert R_1\rVert_{1} \lVert R_2\rVert_{1} $.
\end{lemma}

\begin{proof}
The first statement is obvious. For the first inequality simply use lemma \ref{adjoint inverses} followed by lemma \ref{Alpha Composition}. The second inequality is an easy to prove property of operator norms
\end{proof}

\subsection{Comparing Theorems \ref{Corrupted PAC Bayes} and \ref{Relative Lower Bound}}

What we have shown is the following implication, for all reconstructible $T$
\begin{align*}
\frac{C_1}{\sqrt{n}} \leq \Rbar{L}(e_n) \leq \frac{C_2 \lVert \ell \rVert_\infty}{\sqrt{n}} \Rightarrow \frac{C_1}{\sqrt{\alpha(T) n}} \leq \Rbar{L}(\tilde{e}_n) \leq \frac{C_2 \lVert \tilde{\ell} \rVert_\infty}{\sqrt{n}}.
\end{align*}
By lemma \ref{adjoint inverses}, in the worse case $\lVert \tilde{\ell} \rVert_\infty \geq \frac{\lVert \ell \rVert_\infty}{\alpha(T)}$, and in the ``optimistic worst case" we arrive at bounds a factor of $\alpha(T)$ apart. We do not know if this is the fault of our upper or lower bounding techniques. However, when considering \emph{specific} $\ell$ and $T$ this gap is no longer present (see section \ref{Examples}).

\subsection{Comparing Theorems \ref{Collection of Corrupted PAC Bayes} and \ref{Collection of Corrupted Lower Bound}}

Assuming $c_T$ is the cost of acquiring data corrupted by $T$, theorem \ref{Collection of Corrupted Lower Bound} the ranks the utility of different corruptions by $\frac{1}{\lVert \tilde{\ell} \rVert_\infty^2 c_T}$ where as theorem \ref{Collection of Corrupted Lower Bound} ranks by $\frac{\alpha(T)}{c_T}$. By lemma \ref{adjoint inverses}, $\frac{1}{\alpha(T)}$ is a proxy for $\frac{\lVert \ell \rVert_\infty}{\lVert \tilde{\ell} \rVert_\infty}$ meaning both theorems are ``doing the same thing". In theorems \ref{Collection of Corrupted Lower Bound} and \ref{Collection of Corrupted PAC Bayes} we have best case and a worst case loss specific method for choosing data sets. Theorem \ref{Collection of Corrupted PAC Bayes} combined with 1emma \ref{noisy loss norm} provides a worst case loss insensitive method for choosing data sets. 

\section{What if Clean Learning is Fast?}

The preceding largly solves the problem of learning from corrupted data when learning from the clean distribution occurs at a slow ($\frac{1}{\sqrt{n}}$) rate. The reader is directed to section \ref{fast learning} for some preliminary work on when corrupted learning also occurs at a fast rate. 

\section{Proper Losses and Convexity}

\begin{definition}

A loss $\ell : \obs \times \PP(\obs) \rightarrow \RR$ is \emph{proper} if 
$$
P \in \argmin_{Q \in \PP(\obs)} \EE_{z\dist P} \ell(z,Q).
$$
It is \emph{stricly proper} if $P$ is the \emph{unique} minimizer.

\end{definition}
Proper losses provide suitable surrogate losses for learning problems. All strictly proper losses can be \emph{convexified} through the use of the \emph{canonical} link function \cite{Reid2010,Vernet2011}. Ultimately one works with a loss of the form 
$$
\ell(z,v) = v(z) + \Psi(v) \ind
$$
with $v \in \RR^\obs$, $\ind \in \RR^\obs$ the constant function $\bm{z} = 1$ and $\Psi : \RR^\obs \rightarrow \RR$ a convex function.

\begin{theorem}[Preservation of Convexity]
Let $v \in \RR^\obs$ and $\Psi : \RR^\obs \rightarrow \RR$ be a convex function. Define the loss $\ell(z,v) = v(z) + \Psi(v)$. Then 
$$
\tilde{\ell}(\tilde{z}, v ) = R^*v(\tilde{z}) + \Psi(v).
$$
Furthermore this loss is convex in $v$.
\end{theorem}
This was first noticed in \cite{Cid-Sueiro2014}.

\section{Uses in Supervised Learning}
Recall in supervised learning $\obs = X \times Y$ and the goal is to find a function that predicts $Y$ from $X$ with low expected loss. Many supervised learning techniques proceed by minimizing a proper loss. Given a suitable function class $\Fclass \subseteq \PP(Y)^X$ and a strictly proper loss $\ell$, they attempt to find 
$$
f^* = \argmin_{f \in \Fclass} \EE_{(x,y) \dist P} \ell(y,f(x)).
$$
Using the canonical link function and a careful chosen function class, leaves the learner with a convex problem. If we assume the labels have been corrupted by a corruption $T : Y \rightsquigarrow \tilde{Y}$, we can correct for the corruptions and solve for
$$
\argmin_{f \in \Fclass} \EE_{(x,\tilde{y}) \dist \tilde{P}} \tilde{\ell}(\tilde{y},f(x)).
$$
This objective is equivalent to the first and will also be convex.

\section{Conclusion}

We have sought to solve the problem of how to rank different forms of corrupted data with the ultimate goal of making informed decisions regarding to the acquisition of data sets. To do so we have introduced a general framework in which many corrupted learning tasks can be expressed. Furthermore, we have derived general upper and lower bounds for the reconstructible subset of corrupted learning problems. Finally, we have shown that in some examples these bounds are tight enough to be of use and that they produce the quantities one would expect. These bounds facilitate the ranking of different corrupted data, either through the use of best case lower bounds or worst case upper bounds. We have shown both \emph{loss specific} and \emph{worst case as the loss is varied} bounds. Future work will attempt to further refine these methods as well as extend the framework to non reconstructible problems such as multiple instance learning and learning with label proportions. Theorems \ref{Collection of Corrupted PAC Bayes} and \ref{Collection of Corrupted Lower Bound} provide means of choosing between data sets that feature collections of different corrupted data. 

\newpage

\section{Appendix}

\subsection{Examples}\label{Examples}

We now show examples of common corrupted learning problems. Once again, our focus is corruption of the \emph{labels} and not the \emph{instances}. Thus we work directly with losses $\ell : Y \times A \rightarrow \RR$. In particular we work with classification problems. We present the worst case upper bound, $\lVert R^* \rVert_{\infty}$, as well as the upper bound relevant for $01$ loss, $\ell_{01}$.

\subsubsection{Noisy Labels}

We consider the problem of learning from noisy binary labels \cite{Angluin1988,Natarajan2013}. Here $\sigma_{i}$ is the probability that class $i$ is flipped. We have 
\begin{align*}
T = \TnoisyLabels & \ \ R^* = \MnoisyLabels.
\end{align*}
This yields 
$$
\tilde{\ell}(y,a) = \frac{(1-\sigma_{-y}) \ell(y,a) - \sigma_y \ell(-y,a)}{1- \sigma_{-1} - \sigma_1}.
$$ 
The above equation is lemma 1 in \cite{Natarajan2013} and is the original method of unbiased estimators. Interestingly, even if $\ell$ is positive, $\tilde{\ell}$ can be negative. If the noise is symmetric with $\sigma_{-1}= \sigma_{1} = \sigma$ and $\ell$ is $01$ loss then
$$
\tilde{\ell}(y,a) = \frac{\ell_{01}(y,a) - \sigma}{1 - 2 \sigma}
$$
which is just a rescaled and shifted version of $01$ loss. If we work in the realizable setting, ie there is some $f \in \mathcal{F}$ with
$$
\EE_{(x,y)\dist P} \ell_{01}(y, f(x)) = 0
$$
then the above provides an interesting correspondence between learning with symmetric label noise and learning under distributions with large Tsybakov margin \cite{Audibert2007}. Taking $\sigma = \frac{1}{2} - h$ with $P$ \emph{separable} in turn implies $\tilde{P}$ has Tsybakov margin $h$. This means bounds developed for this setting \cite{Massart2006} can be transferred to the setting of learning with symmetric label noise. Our lower bound reproduces the results of \cite{Massart2006}
\\
\\
Below is a table of the relevant parameters for learning with noisy binary labels. These results directly extend those present in \cite{Kearns1998} that considered only the case of symmetric label noise. 

\begin{center}
\begin{tabular}{l||l}
\multicolumn{2}{l}{Learning with Label Noisy (Binary)} \\ \bottomrule[2pt]
 $T$ & $\TnoisyLabels$  \\ 
\hline $R^*$ & $\MnoisyLabels$ \\ 
\hline $\alpha(T)$ & $|1-\sigma_{-1} - \sigma_1|$ \\ 
\hline $\lVert R^* \rVert_{\infty}$ & $\frac{1}{|1-\sigma_{-1} - \sigma_1|}\max(1 - \sigma_{-1} + \sigma_1, 1-\sigma_1 + \sigma_{-1})$ \\ 
\hline $\lVert \tilde{\ell}_{01} \rVert_{\infty}$  & $\frac{1}{|1-\sigma_{-1} - \sigma_1|}\max(1 - \sigma_{-1} , 1-\sigma_1 , \sigma_{-1} , \sigma_{1})$ \\
\end{tabular}
\end{center}
We see that as long as $\sigma_{-1} + \sigma_{1} \neq 1$ $T$ is reconstructible. The pattern we see in this table is quite common. $\lVert R^* \rVert_{\infty}$ tends to be marginally greater than $\frac{1}{\alpha(T)}$, with $\lVert \tilde{\ell}_{01} \rVert_{\infty}$ less than both. In the symmetric case our lower bound reproduces those of \cite{Aslam1996}.

\subsubsection{Semi-Supervised Learning}

We consider the problem of semi-supervised learning \cite{Chapelle2010}. Here $1-\sigma_i$ is the probability class $i$ has a missing label. We first consider the easier symmetric case where $\sigma_{-1}= \sigma_{1} = \sigma$. 

\begin{center}
\begin{tabular}{l||l}
\multicolumn{2}{l}{Symmetric Semi-Supervised Learning} \\ \bottomrule[2pt]
 $T$ & $\Tsemisymmetric$  \\ 
\hline $R^*$ & $\Msemisymmetric$ \\ 
\hline $\alpha(T)$ & $\sigma$ \\ 
\hline $\lVert R^* \rVert_{\infty}$ & $\frac{1}{\sigma}$ \\ 
\hline $\lVert \tilde{\ell}_{01} \rVert_{\infty}$  & $\frac{1 - 2 \sigma + 2 \sigma^2}{2\sigma + 3\sigma - 5 \sigma^2} $ \\
\end{tabular}
\end{center}
Once again $\lVert \tilde{\ell}_{01} \rVert_{\infty} \leq \frac{1}{\alpha(T)}$. As long as $\sigma \neq 0$. Our lower bound confirms that in general unlabelled data does not help \cite{Balcan2010}. Rather than using the method of unbiased estimators, one could simply throw away the unlabelled data leaving behind $\sigma n$ labelled instances on average.

\begin{center}
\begin{tabular}{l||l}
\multicolumn{2}{l}{Semi-Supervised Learning} \\ \bottomrule[2pt]
 $T$ & $\Tsemi$  \\  
\hline $\alpha(T)$ & $\max_i \sigma_i$ \\ 
\end{tabular}
\end{center}
Other parameters for the more general case are omitted due to complexity (they involve the maximum of three 4th order rational equations). They are available in closed form.

\subsubsection{Three Class Symmetric Label Noise}

In line with \cite{Kearns1998}, here we present parameters for the three class variant of symmetric label noise. We have $\tilde{Y} = Y = \{1,2,3\}$ with $P(\tilde{Y} = \tilde{y} | Y =y ) = 1 -\sigma$, if $y = \tilde{y}$ and $\frac{\sigma}{2}$ otherwise.  

\begin{center}
\begin{tabular}{l||l}
\multicolumn{2}{l}{Learning with Symmetric Label Noisy (Multiclass)} \\ \bottomrule[2pt]
 $T$ & $\TnoisyLabelsmulticlass$  \\ 
\hline $R^*$ & $\MnoisyLabelsmulticlass$ \\ 
\hline $\alpha(T)$ & $|1 - \frac{3}{2} \sigma |$ \\ 
\hline $\lVert R^* \rVert_{\infty}$ & $\frac{2 + \sigma}{|2 - 3 \sigma |}$ \\ 
\hline $\lVert \tilde{\ell}_{01} \rVert_{\infty}$  & $\frac{2}{| 2 - 3 \sigma |} \max(\sigma, 1- \sigma)$ \\
\end{tabular}
\end{center}
We see that as long as $\sigma \neq \frac{2}{3}$ $T$ is reconstructible. Once again $\lVert \tilde{\ell}_{01} \rVert_{\infty} \leq \frac{1}{\alpha(T)}$.

\subsubsection{Partial Labels}

Here we follow \cite{Cour2011} with $Y = \{1,2,3\}$ and $\tilde{Y} = \{0,1\}^{Y}$ the set of partial labels. A partial label of $(0,1,1)$ indicates that the true label is either $2$ or $3$ but not $1$. We assume that a partial label always includes the true label as one of the possibilities and furthermore that spurious labels are added with probability $\sigma$.

\begin{center}
\begin{tabular}{l||l}
\multicolumn{2}{l}{Learning with Partial Labels} \\ \bottomrule[2pt]
 $T$ & $\left(
 \begin{array}{ccc}
  0 & 0 & (1-\sigma )^2 \\
  0 & (1-\sigma )^2 & 0 \\
  0 & (1-\sigma ) \sigma  & (1-\sigma ) \sigma  \\
  (1-\sigma )^2 & 0 & 0 \\
  (1-\sigma ) \sigma  & 0 & (1-\sigma ) \sigma  \\
  (1-\sigma ) \sigma  & (1-\sigma ) \sigma  & 0 \\
  \sigma ^2 & \sigma ^2 & \sigma ^2 \\
 \end{array}
 \right)$ \\ 
\hline $\alpha(T)$ & $1 - \sigma$ \\ 
\end{tabular}
\end{center}
We see that as long as $\sigma \neq 1$ $T$ is reconstructible. In this case $\lVert \tilde{\ell}_{01} \rVert_{\infty}$ and $\lVert R^* \rVert_{\infty}$ are given by more complicated expressions (however they are both available in closed form). We display their interrelation in a graph in below. To the best of our knowledge, there are no upper and lower bounds are present in the literature for this problem.
\begin{center}
\includegraphics[width=0.8\linewidth]{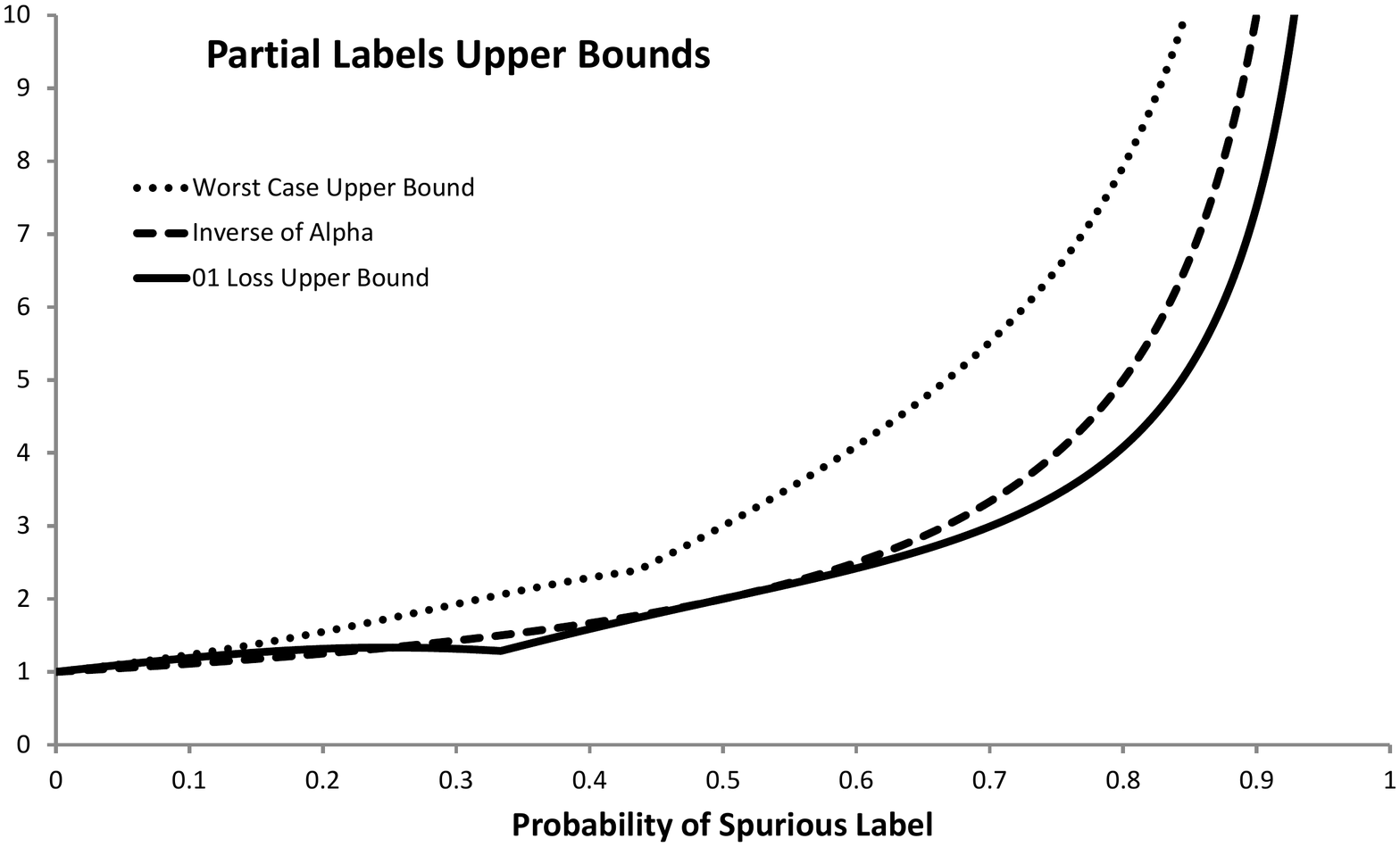}
\end{center}

\subsection{PAC Bayesian Bounds}\label{PAC Bayes Appendix}

PAC Bayesian bounds provide methods to assess the quality of any algorithm $\alg : \obs \rightsquigarrow A$. All of the bounds presented in this section appear in \cite{Zhang2006}. We use the shorthand $\ell(S,a) = \frac{1}{|S|}\sum_{z\in S} \ell(z,a)$.

\begin{theorem}[PAC Bayes]\label{PAC-Bayes}
For all sets $\obs$, $P \in \PP(\obs)$, priors $\pi \in \PP(A)$, algorithms $\alg : \obs \rightsquigarrow A$, functions $L : \obs \times A \rightarrow \RR$ and $\beta > 0$
$$
\EE_{z \dist P} \EE_{a \dist \alg(z)} - \frac{1}{\beta} \log (\EE_{z' \dist P} e^{- \beta L(z',a)} ) \leq \EE_{z \dist P}\left[ L(z,\alg(z)) +\frac{D_{KL}(\alg(z),\pi) }{\beta}  \right].
$$
Furthermore with probability at least $1- \delta$ on a draw $x\dist P$ with $\pi$, $\beta$ and $\alg$ fixed before the draw,
$$
\EE_{a \dist \alg(z)} - \frac{1}{\beta} \log (\EE_{z' \dist P} e^{- \beta L(z',a)} ) \leq L(z,\alg(z)) +\frac{D_{KL}(\alg(z),\pi) + \logdelta}{\beta}.
$$
\end{theorem}
Combined with standard bounds of the cumulant generating function, theorem \ref{PAC-Bayes} leads to useful generalization bounds.

\begin{lemma}\label{cumulantlowerbound}
Let $\phi : \obs \rightarrow [-a,a]$, then for all $\beta >0$ and all $P$
$$
\EE_{P} \phi - \frac{a^2 \beta }{2} \leq -\frac{1}{\beta}\log(\EE_{P} e^{-\beta \phi} ) 
$$
\end{lemma}

\begin{proof}
See appendix A.1 of \cite{Cesa-Bianchi2006}.

\end{proof}

\subsection{Proof of Theorem \ref{Collection of Corrupted PAC Bayes}}

\begin{proof}
Define $L(\tilde{S},a) = \sum_{i=1}^k \sum_{\tilde{z} \in \tilde{S}_i } \tilde{\ell}_i(\tilde{z}_i,a)$, the sum of the corrupted losses on the sample. We have by theorem \ref{PAC-Bayes}
\begin{align*}
\EE_{\tilde{S} \dist Q} \EE_{a \dist \alg(\tilde{S})} - \frac{1}{\beta} \log (\EE_{\tilde{S}' \dist Q} e^{- \beta L(\tilde{S}',a)} ) &\leq \EE_{\tilde{S} \dist Q}\left[ L(\tilde{S},\alg(\tilde{S})) +\frac{D_{KL}(\alg(\tilde{S}),\pi)}{\beta}  \right] \\
\sum\limits_{i=1}^{k} n_i \EE_{\tilde{S} \dist Q} \EE_{a \dist \alg(\tilde{S})} - \frac{1}{\beta} \log (\EE_{\tilde{z} \dist \tilde{P}_i} e^{- \beta \tilde{\ell_i}(\tilde{z},a)} ) &\leq \EE_{\tilde{S} \dist Q}\left[ L(\tilde{S},\alg(\tilde{S})) +\frac{D_{KL}(\alg(\tilde{S}),\pi)}{\beta}  \right]
\end{align*}
where the first line follows from theorem \ref{PAC-Bayes} and the second from properties of the cumulant generating function. Invoking lemma \ref{cumulantlowerbound} yields
$$
\sum\limits_{i=1}^{k}n_i \left(\EE_{\tilde{S} \dist Q} \tilde{\ell}_i(\tilde{P}_i,\alg(\tilde{S})) - \frac{\lVert \tilde{\ell}_i\rVert^2_\infty \beta}{2} \right) \leq \EE_{\tilde{S} \dist Q}\left[ L(\tilde{S},\alg(\tilde{S})) +\frac{D_{KL}(\alg(\tilde{S}),\pi) }{\beta} \right].
$$
As the $T_i$ are reconstructible, 
$$
\EE_{\tilde{S} \dist Q} \ell(P,\alg(\tilde{S})) \leq \frac{1}{n} \EE_{\tilde{S} \dist Q}\left[ L(\tilde{S},\alg(\tilde{S})) +\frac{D_{KL}(\alg(\tilde{S}),\pi)}{\beta}  \right] + \frac{\left(\sum\limits_{i=1}^{k} r_i \lVert \tilde{\ell}_i\rVert^2_\infty \right) \beta}{2}.
$$
Optimizing over $\beta$ yields the desired result.

\end{proof}

\subsection{Le Cam's Method and Minimax Lower Bounds}\label{sec:le-cam's-method-and-minimax-lower-bounds PROOF}
The development here closely follows \cite{Duchi2013} with some streamlining. We consider a general learning problem with unknowns $\Theta$, observation space $\obs$ and loss $L : \Theta \times A \rightarrow \RR$. For any learning algorithm $\alg : \obs \rightsquigarrow \Theta$, we wish to lower bound the max risk 
$$
\sup_{\theta} \EE_{z \dist e(\theta)} L(\theta,\alg(z)).
$$ 
The method proceeds by reducing a general decision problem to an easier binary classification problem. First one considers a supremum over a restricted set $\{\theta_1, \theta_2\}$. Using Markov's inequality we then relate this to the minimum $01$ loss in a particular binary classification problem. Finally one finds a lower bound for this quantity. With $\theta \dist \{\theta_1, \theta_2\}$ meaning $\theta$ is drawn uniformly at random from the set $\{\theta_1, \theta_2\}$, we have
\begin{align}
\sup_{\theta} \EE_{z \dist e(\theta)}\EE_{a \dist \alg(z)} \relloss(\theta,a) &\geq \sup_{\{\theta_1, \theta_2\}} \EE_{z \dist e(\theta)}\EE_{a \dist \alg(z)} \relloss(\theta,a) \nonumber \\
&\geq \EE_{\theta \dist \{\theta_1, \theta_2\}} \EE_{z \dist e(\theta)}\EE_{a \dist \alg(z)} \relloss(\theta,a) \nonumber \\
&\geq \delta \EE_{\theta \dist \{\theta_1, \theta_2\}} \EE_{z \dist e(\theta)}\EE_{a \dist \alg(z)} \pred{\relloss(\theta,a) \geq \delta}.
\end{align}
Recall the \emph{separation} $\rho : \Theta \times \Theta \rightarrow \RR$, $\rho(\theta_1,\theta_2) = \inf_a \relloss(\theta_1,a) + \relloss(\theta_2,a)$. The separation measures how hard it is to act well against both $\theta_1$ and $\theta_2$ simultaneously. We now assume $\rho(\theta_1,\theta_2) > 2 \delta$. Define $f : A \rightarrow \{\theta_1, \theta_2, \text{error}\}$ where $f(a) = \theta_i$ if $\relloss(\theta_i,a) < \delta$ and error otherwise. This function is well defined as if there exists an action $a$ with $\relloss(\theta_1,a) < \delta$ and $\relloss(\theta_2,a) < \delta$ then $\rho(\theta_1,\theta_2) < 2\delta$ a contradiction. Let $\hat{\alg}$ be the classifier that first draws $a\dist \alg(z)$ and then outputs $f(a)$ we have
\begin{align*}
\sup_{\theta} \EE_{z \dist e(\theta)}\EE_{a \dist \alg(z)} \relloss(\theta,a) &\geq \delta \EE_{\theta \dist \{\theta_1, \theta_2\}} \EE_{z \dist e(\theta)}\EE_{\theta' \dist \hat{\alg}(z)} \pred{\theta\neq \theta'}   \\
&\geq \delta \inf_{\hat{\alg} : \obs \rightsquigarrow \Theta} \EE_{\theta \dist \{\theta_1, \theta_2\}} \EE_{z \dist e(\theta)}\EE_{\theta' \dist \hat{\alg}(z)} \pred{\theta\neq \theta'}  \\
&= \delta \left(\frac{1}{2} - \frac{1}{2} V(e(\theta_1), e(\theta_2)) \right)
\end{align*}
where the first line is a rewriting of (1) in terms of the classifier $\hat{\alg}$, the second takes an infimum over all classifiers and the final line is a standard result in theoretical statistics \cite{Reid2009b}. Taking $\delta = \frac{\rho(\theta_1,\theta_2)}{2}$ yields lemma \ref{Le Cam Lemma}.

\subsection{Proof of Lemma \ref{Variational Divergence for product distribitions}}

\begin{proof}
Firstly $V$ is a \emph{metric} on $\PP(\times_{n=1}^k \obs_i)$ \cite{Reid2009b}. Thus 
\begin{align*}
V(\otimes_{i=1}^k P_i, \otimes_{i=1}^k Q_i) &= V(P_1 \otimes (\otimes_{i=2}^k P_i), Q_1 \otimes (\otimes_{i=2}^k Q_i)) \\
&\leq V(P_1 \otimes (\otimes_{i=2}^k P_i) , Q_1 \otimes (\otimes_{i=2}^k P_i)) + V(Q_1 \otimes (\otimes_{i=2}^k P_i) , Q_1 \otimes (\otimes_{i=2}^k Q_i)) \\
&= V(P_1, Q_1) + V(\otimes_{i=2}^k P_i, \otimes_{i=2}^k Q_i)
\end{align*}
where the first line is by definition, the second as $V$ is a metric and the third is easily verified from the definition of $V$. To complete the proof proceed inductively.
\end{proof}

\subsection{Proof of Lemma \ref{Deconstruction and KL Lemma}}

\begin{proof}
\begin{align*}
\Df(T(P),T(Q)) &= \Df(\lambda F(P) + (1-\lambda) G(P),\lambda F(Q) + (1-\lambda) G(Q)) \\
&\leq \lambda \Df(F(P),F(Q)) + (1 - \lambda) \Df(G(P),G(Q)) \\
&= (1 - \lambda) \Df(G(P),G(Q)) \\
&\leq (1 - \lambda) \Df(P,Q)
\end{align*}
Where the first line follows from the definition, the second from the joint convexity of $f$-divergences \cite{Reid2009b}, the third because $F(P)=F(Q)$ and $D_f(P,P) = 0$ and finally the fourth is from the standard data processing inequality \cite{Reid2009b}.

\end{proof}

\subsection{Proof of Lemma \ref{Existence of Decontruction Lemma}}

The proof of the forward implication is lemma 2 of \cite{Boyen1998}. We prove the reverse implication.
\begin{proof}
As this decomposition works for all pairs of distributions we can take $P = \delta_{x_i} = e_i$ and $Q = \delta_{x_j} =e_j$. As $F(P) = F(Q)$ we must have $F_{ki} = F_{kj} = v_k$ for all $k$. As all of the entries of $(1-\lambda) G$ are positive, we have 
$\lambda v_k \leq T_{ki}$ and $\lambda v_k \leq T_{kj}$. Hence $\lambda v_k \leq \min(T_{ki},T_{kj})$. Summing over $k$ and remembering that $F$ is column stochastic gives $\lambda \leq \sum_k \min(T_{k,i}, T_{k,j})$. As $i$ and $j$ are arbitrary we have the desired result.
\end{proof}

\subsection{Proof of Theorem \ref{Collection of Corrupted Lower Bound}}

\begin{proof}
Let
$$
T = \otimes_{i=i}^k T_i^{n_i} = \underbrace{T_1 \otimes \dots \otimes T_1}_{n_1 \ \text{times}} \otimes \underbrace{T_2 \otimes \dots \otimes T_2}_{n_2 \ \text{times}} \dots \otimes \underbrace{T_k \otimes \dots \otimes T_k}_{n_k \ \text{times}}.
$$
One has $T(e_n(\theta)) = T_1(e(\theta))^{n_1}\otimes T_2(e(\theta))^{n_2} \otimes \dots \otimes T_k(e(\theta))^{n_k}$. By lemma \ref{Variational Divergence for product distribitions},
\begin{align*}
V(T(e_n(\theta_1)), T(e_n(\theta_2)) &\leq \sum_{i=1}^k n_i V(T_i(e(\theta_1)), T_i(e(\theta_2)) ) \\ 
&\leq \left( \sum\limits_{i=1}^{k}\alpha(T_i) n_i \right) V(e(\theta_1),e(\theta_2)).
\end{align*}
Now proceed as in the proof of theorem \ref{Relative Lower Bound}.
\end{proof}

\subsection{Proof of Lemma \ref{Alpha Composition}}

\begin{proof}
\begin{align*}
\alpha(T_2 T_1) &= \sup_{P, Q \in \PP(X)} \frac{\lVert T_2 T_1(P) - T_2 T_1(Q)  \rVert_1}{\lVert P - Q \rVert_1} \\
&= \sup_{P, Q \in \PP(X)} \frac{\lVert T_2 T_1(P) - T_2 T_1(Q)  \rVert_1}{\lVert T_1(P) - T_2(Q) \rVert_1} \frac{\lVert T_1(P) - T_2(Q) \rVert_1}{\lVert P - Q \rVert_1} \\ 
&\leq \sup_{P, Q \in \PP(X)} \frac{\lVert T_2 T_1(P) - T_2 T_1(Q)  \rVert_1}{\lVert T_1(P) - T_2(Q) \rVert_1} \sup_{P, Q \in \PP(X)} \frac{\lVert T_1(P) - T_2(Q) \rVert_1}{\lVert P - Q \rVert_1} \\
&\leq \sup_{P, Q \in \PP(Y)} \frac{\lVert T_2 (P) - T_2 (Q)  \rVert_1}{\lVert P - Q \rVert_1} \sup_{P, Q \in \PP(X)} \frac{\lVert T_1(P) - T_2(Q) \rVert_1}{\lVert P - Q \rVert_1} \\
&= \alpha(T_2) \alpha(T_1)
\end{align*}
Where the first line follows from the definitions, the second follows if $T_1(P) \neq T_2(Q)$ and the rest are simple rearrangements. For the final inequality, remember that $\alpha(T) \leq 1$.
\end{proof} 

\subsection{Proof of Lemma \ref{noisy loss norm}}

\begin{proof}
By definition $\lVert \tilde{\ell} \rVert_\infty = \sup_{z,a} |\tilde{\ell}(z,a)| = \sup_a \lVert \tilde{\ell}_a \rVert_\infty$. Hence 
\begin{align*}
\lVert \tilde{\ell} \rVert_\infty &= \sup_a \lVert \tilde{\ell}_a \rVert_\infty \\
&\leq \sup_a \lVert R^*\rVert_\infty \lVert \ell_a \rVert_\infty \\
&= \lVert R^*\rVert_\infty \lVert \ell \rVert_\infty
\end{align*} 
where the second line follows from the definition of the operator norm $\lVert R^*\rVert_\infty$.
\end{proof}

\subsection{Proof of Lemma \ref{adjoint inverses}}

\begin{proof}
Firstly $\lVert R \rVert_{1}  = \lVert R^* \rVert_{\infty}$ \cite{Bernstein2009}. From the definition of $\lVert R \rVert_{1}$ we have

\begin{align*}
\lVert R \rVert_{1} &= \sup_{v \in \RR^Y} \frac{\lVert R v \rVert_{1}}{\lVert v\rVert_{1}} \\
&\geq \sup_{u \in \RR^X} \frac{\lVert R T u \rVert_{1}}{\lVert T u\rVert_{1}} \\
&= \sup_{u \in \RR^X} \frac{\lVert u \rVert_{1}}{\lVert T u\rVert_{1}} \\
&= 1 / \left( \inf_{u \in \RR^X} \frac{\lVert T u\rVert_{1}}{\lVert u \rVert_{1}} \right)
\end{align*}
this proves the first inequality. Recall one of the equivalent definitions of $\alpha(T)$ from section \ref{sec:Variational Alpha}
$$
\alpha(T) = \sup_{v \in S} \frac{\lVert T(v) \rVert_1}{\lVert v \rVert_1}
$$
where $S = \{v \in \RR^X : \sum v_i = 0, v \neq 0\}$. Hence trivially $\inf_{u \in \RR^X} \frac{\lVert T u \rVert_{1}}{\lVert u \rVert_{1}} \leq \alpha(T)$.

\end{proof}

\subsection{Corrupted Learning when Clean Learning is Fast}\label{fast learning}

There are many conditions under which clean learning is fast, here we focus on the Bernstein condition presented in \cite{Erven2012}.

\begin{definition}
Let $P\in \PP(\obs)$, $\ell$ a loss and $a_P = \argmin_a \EE_{z \dist P} \ell(z,a)$. A pair $(\ell, P)$ satisfies the \emph{Bernstein condition} with constant $K$ if for all $a \in A$
$$
\EE_{z \dist P} (\ell(z,a) - \ell(z,a_P))^2 \leq K \ \EE_{z \dist P} \ell(z,a) - \ell(z,a_P)
$$
\end{definition}
When $A$ is finite, such a condition leads to $\frac{1}{n}$ rates of convergence.  From results in \cite{Zhang2006} we have the following theorem.

\begin{theorem}[PAC Bayes Bernstein]\label{PAC Bayes Bernstein}
Let $\gamma = \frac{(e^\beta - 1 - \beta)}{\beta  \lVert \ell \rVert_{\infty}}$. For all $P$, priors $\pi$, algorithms $\alg$, bounded losses $\ell$ and $\beta > 0$
$$
\EE_{S\dist P^n} \left[ \ell(P,\alg(S)) - \gamma \ell^2(P,\alg(S)) \right]\leq \EE_{S\dist P^n}\left[\ell(S,\alg(S))  +  \lVert \ell \rVert_{\infty}\left(\frac{D_{KL}(\alg(S),\pi)}{\beta n} \right)\right] .
$$
Furthermore with probability at least $1-\delta$ on a draw $S \dist P^n$ with $\alg$, $\beta$ and $\pi$ chosen before the draw
$$
\ell(P,\alg(S)) - \gamma \ell^2(P,\alg(S)) \leq \left[\ell(S,\alg(S))  +  \lVert \ell \rVert_{\infty}\left(\frac{D_{KL}(\alg(S),\pi) + \logdelta}{\beta n}  \right) \right]. 
$$
\end{theorem}
We are now in a position to show that the Bernstein condition leads to fast rates for ERM.
\begin{theorem}(Fast Rates for ERM)
Let $\alg$ be ERM with $A$ finite. If $(\ell,P)$ satisfies the Bernstein condition then for some constant $C$
$$
\EE_{S \dist P^n} \ell(P,\alg(S)) - \ell(P,a_P) \leq \frac{C \log(|A|)}{n}.
$$
Furthermore with probability at least $1- \delta$ on a draw from $P^n$ one has 
$$
\ell(P,\alg(S)) - \ell(P,a_P) \leq \frac{C\left(\log(|A|) + \logdelta\right)}{n}.
$$
\end{theorem}

\begin{proof}
First, define $\ell_P(z,a) = \ell(z,a) - \ell(z,a_P)$. $l_P$ measures the loss relative to the best action for the distribution $P$. It is easy to verify that for bounded $\ell$, $\lVert \ell_P \rVert_\infty \leq 2 \lVert \ell \rVert_\infty$. We now utilize theorem \ref{PAC Bayes Bernstein} with $\ell_P$ and $\pi$ uniform on $A$. This yields
$$
\EE_{S\dist P^n} \left[ \ell_P(P,\alg(S)) - \gamma \ell_P^2(P,\alg(S)) \right]\leq \frac{1}{n} \EE_{S\dist P^n}\left[\ell_P(S,\alg(S))  + \lVert \ell_P \rVert_{\infty}\left(\frac{\log(|A|)}{\beta} \right)\right]
$$
with $\gamma = \frac{(e^\beta - 1 - \beta)}{\beta  \lVert \ell_P \rVert_{\infty}}$. Firstly ERM minimizes the right hand side of the bound meaning
$$
\frac{1}{n} \EE_{S\dist P^n}\left[\ell_P(S,\alg(S))  + \lVert \ell_P \rVert_{\infty}\left(\frac{\log(|A|)}{\beta} \right)\right] \leq \frac{1}{n} \left[\lVert \ell_P \rVert_{\infty}\left(\frac{\log(|A|)}{\beta} \right)\right].
$$
To see this consider the algorithm that always outputs $a_P$, this algorithm generalizes very well however it may be suboptimal on the sample. Secondly $(\ell,P)$ satisfies the Bernstein condition with constant $K$. Therefore
$$
(1 - \gamma K) \EE_{S\dist P^n} \ell_P(P,\alg(S)) \leq \frac{1}{n} \left[\lVert \ell_P \rVert_{\infty}\left(\frac{\log(|A|)}{\beta} \right)\right].
$$
Finally chose $\beta$ small enough so that $\gamma K \leq 1$. This can always be done as $\gamma \rightarrow 0$ as $\beta \rightarrow 0_+$. The high probability version proceeds in a similar way.

\end{proof}
A natural question to ask is when does $(\tilde{\ell}, \tilde{P} )$ satisfy the Bernstein condition?

\begin{theorem}\label{Noisy Bernstein}
If $(\tilde{\ell}, \tilde{P} )$ satisfies the Bernstein condition with constant $K$ then $(\ell,P)$ also satisfies the Bernstein condition with the same constant.
\end{theorem}

\begin{proof}
\begin{align*}
K \EE_{z \dist P} \ell(z,a) - \ell(z,a_P) &= K \EE_{\tilde{z} \dist \tilde{P}} \tilde{\ell}(z,a) - \tilde{\ell}(z,a_P) &\\
&\geq \EE_{\tilde{z} \dist \tilde{P}} (\tilde{\ell}(\tilde{z},a) - \tilde{\ell}(\tilde{z},a_P))^2 \\
&= \EE_{z \dist P}\EE_{\tilde{z} \dist T(z)} (\tilde{\ell}(\tilde{z},a) - \tilde{\ell}(\tilde{z},a_P))^2 \\
&\geq\EE_{z \dist P} (\EE_{\tilde{z} \dist T(z)}\tilde{\ell}(\tilde{z},a) - \EE_{\tilde{z} \dist T(z)}\tilde{\ell}(\tilde{z},a_P))^2 \\
&= \EE_{z \dist P} (\ell(z,a) - \ell(z,a_P))^2
\end{align*}
where the first line follows from the definition of $\ell$ and because $a_{P} = a_{\tilde{P}}$, the second as $(\tilde{\ell}, \tilde{P} )$ satisfies the Bernstein condition and finally we have used the convexity of $f(x) = x^2$.

\end{proof}
This theorem (almost) rules out pathological behaviour where ERM learns quickly from corrupted data and yet slowly for clean data. At present it is unknown if the converse to theorem \ref{Noisy Bernstein} is true, with the same or possibly different constant. Here we present a partial converse.

\begin{definition}
Let $T : \obs \rightsquigarrow \tilde{\obs}$ be a Markov kernel and $\ell$ a loss. A pair $(\ell,T)$ are \emph{$\eta$-compatible} if for all $z \in \obs$ and $a_1, a_2 \in A$
$$
\EE_{\tilde{z} \dist T(z)} (\tilde{\ell}(\tilde{z},a_1) - \tilde{\ell}(\tilde{z},a_2))^2 \leq \eta (\ell(z,a_1) - \ell(z,a_2))^2.
$$
\end{definition}

\begin{theorem}
If the pair $(\ell,P)$ satisfies the Bernstein condition with constant $K$ and the pair $(\ell,T)$ are $\eta$-compatible then $(\tilde{l},\tilde{P})$ satisfies the Bernstein condition with constant $\eta K$.
\end{theorem}

\begin{proof}
\begin{align*}
\EE_{\tilde{z} \dist \tilde{P}} (\tilde{\ell}(\tilde{z},a) - \tilde{\ell}(\tilde{z},a_P))^2 & = \EE_{z \dist P}\EE_{\tilde{z} \dist T(z)} (\tilde{\ell}(\tilde{z},a) - \tilde{\ell}(\tilde{z},a_P))^2 \\
&\leq \eta  \EE_{z \dist P} (\ell(z,a) - \ell(z,a_P))^2 \\
&\leq \eta K \EE_{z \dist P} \ell(z,a) - \ell(z,a_P) \\
&= \eta K \EE_{\tilde{z} \dist \tilde{P}} \tilde{\ell}(\tilde{z},a) - \tilde{\ell}(\tilde{z},a_P) 
\end{align*}
where we have first used $\eta$-compatibility, then the fact that $(\ell,P)$ satisfies the Bernstein condition with constant $K$ and finally the definition of $\tilde{\ell}$.

\end{proof}
While by no means the final line in fast corrupted learning, this theorem does allow one to prove interesting results in the binary classification setting.

\begin{theorem}\label{Bernstein Label Noise}
Let $T$ be label noise, $T = \TnoisyLabels$, then the pair $(\ell_{01}, T)$ is $\eta$-compatible with $\eta = \max(\left(\frac{1+\sigma_{-1} - \sigma_1}{1-\sigma_{-1} - \sigma_1}\right)^2, \left(\frac{1+\sigma_{1} - \sigma_{-1}}{1-\sigma_{-1} - \sigma_1}\right)^2)$.
\end{theorem}

\begin{proof}
Due to the symmetry of the left and right hand sides of the Bernstein condition, one only needs to check the case where $a_1 = 1$, $a_2 = -1$. Recall 
\begin{align*}
\tilde{\ell}_{01}(\tilde{y},a) &= \frac{(1-\sigma_{-y}) \ell_{01}(\tilde{y},a) - \sigma_y \ell_{01}(-\tilde{y},a)}{1- \sigma_{-1} - \sigma_1} \\
&= \frac{(1-\sigma_{-y} + \sigma_y)\ell_{01}(\tilde{y},a) - \sigma_y}{1- \sigma_{-1} - \sigma_1}.
\end{align*}
For $y = 1$ it is easy to confirm $\left(\ell_{01}(1,1) - \ell_{01}(1,-1) \right)^2 = 1$. We have
\begin{align*}
\tilde{\ell}_{01}(\tilde{y},1) - \tilde{\ell}_{01}(\tilde{y},-1) &= \frac{(1-\sigma_{-y} + \sigma_y)(\ell_{01}(\tilde{y},1)- \ell_{01}(\tilde{y},-1))}{1- \sigma_{-1} - \sigma_1} \\
&= \frac{-\tilde{y}(1-\sigma_{-y} + \sigma_y)}{1- \sigma_{-1} - \sigma_1}.
\end{align*}
Squaring, taking maximums and finally expectations yields the desired result.

\end{proof}
One very useful example of a pair $(P,\ell)$ satisfying the Bernstein condition with constant $1$ is when $P$ is separable, $\ell$ is $01$ loss and the Bayes optimal classifier is in the function class. Theorem \ref{Bernstein Label Noise} guarantees that in such a setting one can learn at a fast rate from noisy examples.

\newpage

\bibliographystyle{plain}
\bibliography{ref}

\end{document}